\documentclass{article} % For LaTeX2e
\usepackage{iclr2026_conference,times}
\usepackage{etoolbox}

\makeatletter
\newif\ificlrpreprint
\iclrpreprintfalse
\def\iclrpreprintcopy{\iclrpreprinttrue}

\AtBeginDocument{%
  \ificlrpreprint
    \iclrfinaltrue
    \patchcmd{\@maketitle}{\lhead{Published as a conference paper at ICLR 2026}}%
                           {\lhead{Under review as a conference paper at ICLR 2026}}{}{}%
  \fi
}
\makeatother
\usepackage{amsmath,amsfonts,bm}

\def\eqref#1{equation~\ref{#1}}
\def\Eqref#1{Equation~\ref{#1}}
\def\1{\bm{1}}

\DeclareMathAlphabet{\mathsfit}{\encodingdefault}{\sfdefault}{m}{sl}
\SetMathAlphabet{\mathsfit}{bold}{\encodingdefault}{\sfdefault}{bx}{n}

\newcommand{\E}{\mathbb{E}}

\newcommand{\R}{\mathbb{R}}

\DeclareMathOperator*{\argmin}{arg\,min}

\usepackage{hyperref}
\usepackage{url}

\usepackage{pifont}

\newcommand{\cmark}{\ding{51}}%
\newcommand{\xmark}{\ding{55}}

\title{Efficient Probabilistic Tensor Networks}

\author{%
Marawan Gamal Abdel Hameed \\
Mila \& DIRO, Université de Montréal \\
\texttt{marawan.gamal@mila.quebec} \\
\And
Guillaume Rabusseau \\
Mila \& DIRO, Université de Montréal \\
}

\usepackage{tensor}
\usepackage{amsmath}
\usepackage{amssymb}
\usepackage{amsthm}

\usepackage{booktabs}
\usepackage{float}
\usepackage{subcaption}
\usepackage{graphicx}
\usepackage{multirow}
\usepackage{tablefootnote}
\usepackage{wrapfig}
\usepackage{thmtools}

\renewcommand{\vec}[1]{\MakeLowercase{\mathbf{#1}}}
\newcommand{\mat}[1]{\MakeUppercase{\mathbf{#1}}}

\let\tensorp\tensor
\renewcommand{\tensor}[2][]{
\tensorp{\boldsymbol{\mathcal{\MakeUppercase{#2}}}}{#1}
}

\def \E{\mathbb{E}}
\def \R{\mathbb{R}}
\def \N{\mathbb{N}}

\usepackage{algorithm}
\usepackage{algpseudocode}

\usepackage{tikz}
\usetikzlibrary{positioning}
\usepackage{xcolor} % for colors
\usepackage{tcolorbox}

\newtheorem{lemma}{Lemma}

\newcommand{\dtilde}[1]{\dot{\tilde{#1}}}
\iclrpreprintcopy % Uncomment for camera-ready version, but NOT for submission.
\begin{document}

\maketitle

% \begin{abstract}
% Tensor networks (TNs) enable compact representations of large tensors through shared parameters. Their use in probabilistic modeling is particularly appealing, as TN-based models allow for tractable computation of marginals.
% However, optimizing probabilistic tensor networks (PTNs) can be expensive and numerically unstable. Existing approaches have been developed for Matrix Product States, but are are computationally demanding, not easily parallelizable and are not fully compatible automatic differentiation frameworks. In this work, we propose a conceptually simple approach for stable learning of PTNs that is broadly applicable across diverse TN architectures, easily parallelizable and compatible with automatic differentiation.
% \end{abstract}

\begin{abstract}
Tensor networks (TNs) enable compact representations of large tensors through shared parameters. Their use in probabilistic modeling is particularly appealing, as probabilistic tensor networks (PTNs) allow for tractable computation of marginals. However, existing approaches for learning parameters of PTNs are either computationally demanding and not fully compatible with automatic differentiation frameworks, or numerically unstable. In this work, we propose a conceptually simple approach for learning PTNs efficiently, that is numerically stable. We show our method provides significant improvements in time and space complexity, achieving 10× reduction in latency for generative modeling on the MNIST dataset. Furthermore, our approach enables learning of distributions with 10× more variables than previous approaches when applied to a variety of density estimation benchmarks. Our code is publicly available at \href{https://github.com/marawangamal/ptn}{github.com/marawangamal/ptn}.
\end{abstract}

\section{Introduction}
\label{sec:intro}
% Motivation: Why is this problem important?
Generative modeling has seen widespread adoption in recent years, particularly in the areas of language modeling~\citep{openai2023gpt4}, image and video generation~\citep{ho2022imagenvideo}, drug discovery~\citep{segler2018drugs} and material science~\citep{menon2022generative}. These achievements have been made possible by deep neural network based architectures such as Generative Pretrained Transformers~\citep{radford2018improving}, Generative Adversarial Networks~\citep{goodfellow2014generative}, Variational Auto-encoders (VAEs)~\citep{kingma2014vae}, Normalizing Flows~\citep{rezende2015variational, papamakarios2019normalizing}
and Diffusion models~\citep{ho2020ddpm}.

While generative models used for these applications are remarkably performant in terms of sampling, they fall short in terms of inference. For instance, consider a set of random variables $Y_1, \ldots, Y_N$ and a density $p(Y_1, \ldots, Y_N)$ represented by one of the aforementioned models. Queries such as $p(Y_a | Y_c)$ cannot be performed, where $Y_a, Y_b, Y_c$ are obtained by splitting $(Y_1, \ldots Y_N)$ into three disjoint sets. This is ultimately due to the underlying probabilistic models having intractable marginals~\citep{bond2021deep}.

% Background: What has been done before?
Tensor networks have been proposed for generative modeling since they allow for tractable  marginalization, enabling inference of sophisticated queries such as $p(Y_a | Y_b)$~\citep{han2018unsupervised, miller2021tensor}. Motivated by their success in representing many-body quantum states~\citep{schollwock2011dmrg,orus2014tensor}, matrix product states (MPS) in particular have been investigated for probabilistic modeling~\citep{han2018unsupervised, vieijra2022generative, glasser2019expressive}. In~\cite{glasser2019expressive}, MPS-based models such as Non-Negative Matrix Product States and Born Machines are used for probabilistic modeling. The parameters of these models are learned by minimizing the negative log-likelihood using stochastic gradient descent (SGD). However, this approach does not scale beyond a small number of MPS cores, thereby limiting the number of random variables that can be represented jointly. As shown in Figure~\ref{fig:crown-jewel}d, systems with 100 cores or more result in numerical overflows after only two iterations. We analyze this behavior theoretically and show that for Non Negative Matrix product States instability arises due to exponential growth in the magnitude of the expected value of the tensor entries with an increasing number of cores. Meanwhile, for Born Machines  it is due to exponential growth in the variance of tensor entries with an increasing number of cores.

Alternatively, \cite{han2018unsupervised} and \cite{cheng2019ttn} use the Density Matrix Renormalization Group (DMRG) algorithm~\citep{schollwock2011dmrg} to learn the model parameters of MPS-based models. While this approach stabilizes the computation of tensor elements due to the isometry of MPS cores and enables adaptive learning of MPS ranks, it has a number of downsides in practice. 
 % Condensed to three points:
First, it is computationally demanding in both space and time, as each parameter update requires performing SVD (Singular Value Decomposition) on a materialized fourth-order tensor as depicted in Figures~\ref{fig:crown-jewel}a and~\ref{fig:crown-jewel}d. Second, it is not fully compatible with automatic differentiation as the DMRG algorithm does not provide a differentiable loss function that can be used for end-to-end model training as shown in Figure~\ref{fig:crown-jewel}a. Third, it is not easily parallelizable across the sequence dimension as updating more than two cores at a time would break the canonical form. While methods for parallelization of DMRG exist, they are even more memory intensive as they require the materialization of many fourth-order tensors simultaneously~\citep{Stoudenmire2013PRB} 
% (Figure~\ref{fig:crown-jewel}d reports latency and memory usage of the less memory intensive non-parallel variant of DMRG). 
Lastly, we point out that DMRG implementations are non-trivial and require careful maintenance of a cache, 
which increases the barrier to entry for experimentation with PTNs.
% in order to avoid recomputing intermediate values which would further exacerbate the runtime latency, increasing the barrier to entry for experimentation with PTNs.

% These approaches use the MPS to parametrize \emph{Born Machines} and learn the MPS parameters through the use of the Density Matrix Renormalization Group (DMRG) algorithm~\citep{schollwock2011dmrg}, whilst keeping the MPS in \emph{canonical form}.
% , we refer to this approach as $\mathrm{MPS}_\mathrm{BM+DMRG}$. 
% While the approach of~\cite{han2018unsupervised} enables adaptive learning of MPS ranks, and stabilizes the computation of tensor elements due to the isometry of MPS cores, it has a number of downsides in practice.

% define a reusable style for algorithm boxes
\tcbset{
  algobox/.style={
    colback=blue!5,   % background color
    colframe=blue!40, % frame color
    boxrule=0.4pt,
    arc=2mm,
    left=2mm,
    right=2mm,
    top=1mm,
    bottom=1mm,
  }
}

% \begin{figure}
%     \centering
%     \includegraphics[width=0.5\linewidth]{figures/mps-profile.png}
%     \caption{Enter Caption}
%     \label{fig:placeholder}
% \end{figure}

\begin{figure}[t]
\centering
\setlength{\fboxsep}{0pt} % no padding inside frame boxes

% LEFT column (stacked: top + bottom)
\begin{minipage}[t]{0.54\textwidth}
\vspace{0pt}

% Left Top: Algorithm 1 (MPS-BM)
\begin{tcolorbox}[algobox]
  \input{diagrams/mps-bm}
\end{tcolorbox}
\par\smallskip\centerline{ (a) DMRG~\citep{han2018unsupervised}}

\vspace{1em} % vertical gap between top and bottom

% Left Bottom: Algorithm 2 (MPS-SGD)
\begin{tcolorbox}[algobox]
  \input{diagrams/mps-sigma}
\end{tcolorbox}
\par\smallskip\centerline{(c) LSF (ours)}

\end{minipage}
\hfill
% 
% RIGHT column (stacked: top + bottom)
\begin{minipage}[t]{0.45\textwidth}
\vspace{0pt}

% Right Top: Algorithm 3 (MPS-σ LSF)
\begin{tcolorbox}[algobox]
  \input{diagrams/mps-sgd}
\end{tcolorbox}
\par\smallskip\centerline{ (b) SGD}

\vspace{1em} % vertical gap between top and bottom

% Right Bottom: Profile figure
\includegraphics[width=\linewidth]{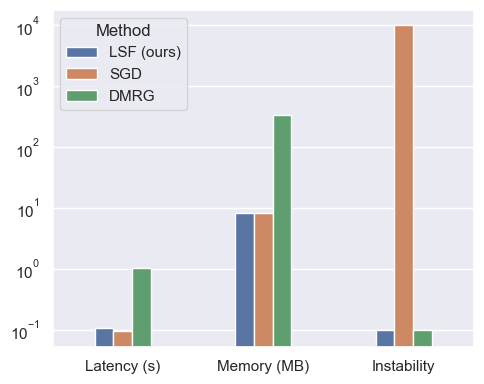}
\par\smallskip\centerline{(d) Performance}

\end{minipage}

\caption{Comparison between training methods for PTNs. 
(a) DMRG~\citep{han2018unsupervised}, (b) SGD~\citep{glasser2019expressive}, (c) our method using SGD with logarithmic scale factors (LSF) and (d) latency, memory usage and a measure of instability of the methods. DMRG has exponentially higher latency and memory usage compared with LSF and SGD. However, SGD is numerically unstable. The instability metric is equal to the remaining iterations out of $10^4$ when a numerical overflow is encountered. Even with a modest system size of 100 cores, numerical overflow occurs after just two iterations (see~\ref{app:crown-jewel-hps} for experimental details).}
\label{fig:crown-jewel}
\end{figure}

In this work, we propose a conceptually simple approach for learning PTNs that (i) is numerically stable, (ii) achieves significant improvements in both space and time complexity compared with DMRG as shown in Figure~\ref{fig:crown-jewel}d and (iii) is fully compatible with automatic differentiation, simplifying its integration into standard machine learning frameworks as shown in Figure~\ref{fig:crown-jewel}c. 

In summary, our contributions are:
\begin{itemize}
    \item Theoretically analyzing the cause of numerical instability when learning parameters of MPS-based PTNs using stochastic gradient descent and providing lower bounds that characterize the instability of Non Negative Matrix Product states and Born Machines.
    \item Developing a numerically stable method for computing the negative log-likelihood through the use logarithmic scale factors.
   \item Demonstrating that on real data our method can be used to process sequences that are 10× longer than those managed by previous SGD based approaches. Meanwhile, being 10× faster than alternative numerically stable methods relying on DMRG.
\end{itemize}

\section{Related Work}
\textbf{Tensor Networks} have been used to approximate high dimensional tensors in a variety of domains including neuroscience~\citep{williams2018unsupervised, cong2015tensor}, chemistry~\citep{murphy2013fluorescence} and hyperspectral imaging~\citep{fang2017cp}. Classical learning methods include the Higher-Order Singular Value Decomposition (HOSVD) for Tucker models~\citep{kolda2009tensor}, Alternating Least Squares (ALS) for Canonical Polyadic (CP) decompositions~\citep{carroll1970analysis}, and the Tensor-Train SVD (TT-SVD) for tensor-train (TT) formats~\citep{oseledets2011tensor}. In quantum many-body physics, the Density Matrix Renormalization Group (DMRG) provides a powerful scheme for optimizing matrix product states (MPS) through local updates~\citep{schollwock2011density}, enabling adaptive bond dimensions.

\textbf{Probabilistic modeling with Tensor Networks (TNs)} has been investigated for uni-variate conditional distributions~\citep{novikov2017exponential, stoudenmire2016supervised}, multi-variate distributions~\citep{han2018unsupervised, vieijra2022generative, cheng2019ttn, glasser2019expressive}, as well as sequence modeling tasks~\citep{miller2021tensor}. In~\cite{han2018unsupervised, cheng2019ttn}, the authors rely on sequential DMRG, although a parallelizable variant of DMRG had been proposed~\citep{Stoudenmire2013PRB}. While the parallelizable variant of DMRG has been shown to nearly reach the expected theoretical speedup, it requires substantially more memory during training as it results in the materialization of many fourth-order tensors \emph{simultaneously}. Finally, natural gradient descent has also been proposed for training MPS-based Born Machines, to avoid local minima in Quantum State Tomograpy~\citep{tang2025initialization}.

\textbf{Relationships between PTNs and alternative probabilistic modeling frameworks} such as Probabilistic Graphical Models (PGMs) and Probabilistic Circuits (PCs) have been previously investigated. In~\cite{glasser2019expressive}, mappings between hidden markov models and non negative MPS-based distributions have been provided, as well as mappings between quantum circuits and MPS-based born machines. Furthermore, in~\cite{loconte2025relationship}, the Tucker decomposition and MPS have been shown to have equivalent shallow and deep probabilistic circuit representations, respectively. Lastly,~\cite{miller2021probabilistic} provides a hybrid framework for PGMs and PTNs.

\section{Method}

We consider the task of modeling multi-variate  distributions of the form
\begin{equation}
\label{eq:joint-dist}
p(y_1, y_2, \ldots, y_N),
\end{equation}
where $y_i \in \mathcal{Y}_N$ is a discrete random variable. Since a direct representation of~\Eqref{eq:joint-dist} is generally intractable and learning in such high-dimensional spaces is hindered by the curse of dimensionality, one typically resorts to parametric approaches. Matrix Product States is a class of parametric models that can represent such distributions with the added benefit that marginals are tractable to compute.

The rest of this section is organized as follows: Section~\ref{sec:method--mps} introduces the Matrix Product State (MPS) model. Sections~\ref{sec:method--ptn-mps-bm} \& \ref{sec:method--ptn-mps-sigma} define MPS-based probabilistic tensor networks (PTNs). Section~\ref{sec:method--learning-mps-params}  outlines the trade-offs between using DMRG and SGD to train MPS-based models and provides a theoretical analysis of the stability issue encountered when using SGD to train PTNs. Section~\ref{sec:method--lsf} introduces our method using logarithmic scale factors. Section~\ref{sec:method--automatic-differentiation} compares our proposed method with the Density Matrix Renormalization Group (DMRG) for learning MPS-based models, in terms of compatibility with automatic differentiation. Lastly, Section~\ref{sec:method--sampling} demonstrates how sampling can be performed using MPS-based probabilistic models.

\subsection{Matrix Product States}
\label{sec:method--mps}
The Matrix Product State (MPS) model provides a structured representation of high-order tensors by factorizing them into a sequence of matrices. 
Thus, tensor $\tensor{T} \in \mathbb{R}^{D_1 \times D_2 \times \cdots \times D_N}$
can be approximated as a sequence of matrix multiplications
\begin{equation}
\label{eq:mps-factorization}
\tensor{T}_{y_1\cdots y_h}
\approx  \tensor{G}^{(1)}[y_1] \cdots \tensor{G}^{(N)}[y_N] 
\end{equation}
where $\tensor{G}^{(i)} \in \mathbb{R}^{R_i \times D_i \times R_{i+1}}$ are referred to as the MPS \emph{cores}, $R_i$ is referred to as the \emph{Ranks} or \emph{Bond Dimensions}, $D_i$ are referred to as the \emph{input dimensions}, $[\cdot]$ indicates slicing along the input dimension (i.e., $\tensor{g}^{(i)}[k] \in \R^{R_i \times R_{i+1}}$), and the boundaries are constrained such that $R_i = r_N = 1$, making the contraction in~\Eqref{eq:mps-factorization} scalar valued. Assuming equal dimensions $D_i=D$ for all $i$, the MPS parameterization has a space complexity of $\mathcal{O}(NDR^2)$ which is \emph{linear} in $N$, compared with $\mathcal{O}(D^N)$ in the original tensor.

\subsection{Probabilistic Modeling with $\mathrm{MPS}_\mathrm{BM}$}
\label{sec:method--ptn-mps-bm}
In order to represent a valid probability distribution using the MPS model, the parameters must be constrained such that all the tensor entries are positive and sum to one. Born Machines enforce such constraints by taking inspiration from quantum mechanics where wavefunctions induce probability distributions described by the squared norm of the wavefunction
\begin{equation}
    \label{eq:mps-born}
    p(y_1, \ldots, y_N) = \frac{|\Psi_\mathrm{BM}(\vec{y})|^2}{Z}, \quad \Psi_\mathrm{BM}(\vec{y}) = \tensor{G}^{(1)}[y_1] \cdots \tensor{G}^{(N)}[y_N]
\end{equation}
where $y_i \in \mathcal{Y}_i \subset \mathbb{N}$ and $\vec y = (y_1,\ldots, y_N)$. At first glance, it may seem that computing the normalization constant $Z$ in~\Eqref{eq:mps-born} requires a summation over an exponential number of terms, as it would require summing up all the squares of the elements in the underlying tensor
\begin{equation}
    \label{eq:mps-normalization-constant}
    Z = \sum_{y_1^\prime, \ldots y_N^\prime \in \mathcal{Y}^N} 
    \left(\tensor{g}^{(1)}[y^\prime_1] \cdots \tensor{g}^{(N)}[y^\prime_N] 
    \right)^2.
\end{equation}
Remarkably, a key property of the underlying MPS model is the ability to compute the normalization constant efficiently, in time linear in $N$. Algebraically, the computation simplifies to
\begin{align*}
\label{eq:mps-born-norm-const}
Z 
% &= \sum_{\substack{r_1, \ldots, r_{N+1} \\ y^\prime_1, \ldots, y^\prime_N}} 
% \left(\tensor{g}^{(1)}_{r_1, r_2}[y^\prime_1] \tensor{g}^{(2)}_{r_2, r_3}[y^\prime_2] \cdots \tensor{g}^{(N)}_{r_N, r_{N+1}}[y^\prime_N]\right)^2 \\
% &= \sum_{\substack{r_1, \ldots, r_{N+1} \\ y^\prime_1, \ldots, y^\prime_N}} 
% \tensor{g}^{(1)}_{r_1, r_2}[y^\prime_1] \tensor{g}^{(1)}_{r_1, r_2}[y^\prime_1] \tensor{g}^{(2)}_{r_2, r_3}[y^\prime_2] \tensor{g}^{(2)}_{r_2, r_3}[y^\prime_2] \cdots \tensor{g}^{(N)}_{r_N, r_{N+1}}[y^\prime_N] \tensor{g}^{(N)}_{r_N, r_{N+1}}[y^\prime_N] \\
&= \sum_{r_1, r_2, y^\prime_1} 
\tensor{g}^{(1)}_{r_1, r_2}[y^\prime_1] \tensor{g}^{(1)}_{r_1, r_2}[y^\prime_1]
\;\;\cdots 
\sum_{r_N, r_{N+1}, y^\prime_N} \tensor{g}^{(N)}_{r_N, r_{N+1}}[y^\prime_N] \tensor{g}^{(N)}_{r_N, r_{N+1}}[y^\prime_N].
\end{align*}
This property is easy to see using tensor network diagrams as depicted in Figure~\ref{fig:mps-norm}. Notably, this property is not unique to MPS and other tensor network structures such as Canonical Polyadic (CP) and Tensor Tree exhibit similar simplifications as shown in~\cite{cheng2019ttn}.

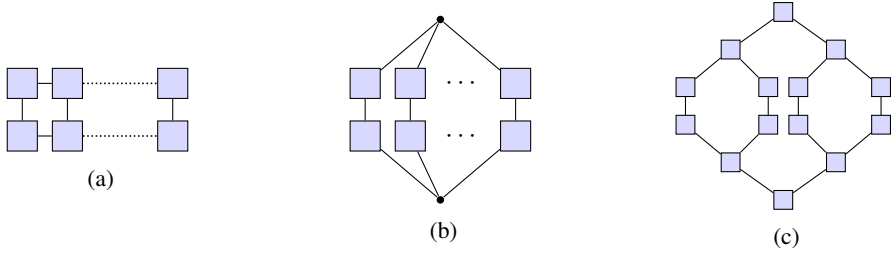
\begin{figure}[tb]
  \centering
  % (a)
  \begin{subfigure}[t]{0.32\linewidth}
    \centering
    \begin{tikzpicture}[baseline=-0.5ex,
  box/.style={draw, fill=blue!15, minimum size=4mm, inner sep=0pt},
  dotline/.style={densely dotted, line width=0.6pt}
]
  % left column
  \node[box] (t1) at (0,  0.35) {};
  \node[box] (b1) at (0, -0.35) {};
  \draw (t1) -- (b1);

  \node[box] (t2) at (0.6,  0.35) {};
  \node[box] (b2) at (0.6, -0.35) {};
  \draw (t2) -- (b2);

  \node[box] (tN) at (2,  0.35) {};
  \node[box] (bN) at (2, -0.35) {};
  \draw (tN) -- (bN);

  % Lines
  \draw (t1) -- (t2);
  \draw (b1) -- (b2);

  % dotted horizontal connections (ellipsis)
  \draw[dotline] (t2.east) -- (tN.west);
  \draw[dotline] (b2.east) -- (bN.west);
\end{tikzpicture}
    \caption{}
    \label{fig:mps-norm}
  \end{subfigure}
  % (b)
  \begin{subfigure}[t]{0.32\linewidth}
    \centering
    \begin{tikzpicture}[baseline=-0.5ex,
  box/.style={draw, fill=blue!15, minimum size=4mm, inner sep=0pt},
  dotline/.style={densely dotted, line width=0.6pt},
  dot/.style={circle, fill=black, inner sep=1pt} % <-- dot style
]

  % Top
  \node[dot] (r1) at (1,  1.2) {};
  \node[dot] (r2) at (1,  -1.2) {};

  % left column
  \node[box] (t1) at (0,  0.35) {};
  \node[box] (b1) at (0, -0.35) {};
  \draw (t1) -- (r1);
  \draw (b1) -- (r2);
  \draw (t1) -- (b1);

  \node[box] (t2) at (0.6,  0.35) {};
  \node[box] (b2) at (0.6, -0.35) {};
  \draw (t2) -- (r1);
  \draw (b2) -- (r2);
  \draw (t2) -- (b2);

  \node[box] (tN) at (2,  0.35) {};
  \node[box] (bN) at (2, -0.35) {};
  \draw (tN) -- (r1);
  \draw (bN) -- (r2);
  \draw (tN) -- (bN);

  \node (ellipsis) at (1.3,  0.35) {$\cdots$};
  \node (ellipsis) at (1.3,  -0.35) {$\cdots$};

\end{tikzpicture}
    \caption{}
    \label{fig:cp-norm}
  \end{subfigure}
  % (b)
  \begin{subfigure}[t]{0.32\linewidth}
    \centering
    % Define tree command with leaf connection points
\newcommand{\binarytree}[3]{%
  % #1 = x offset
  % #2 = y offset  
  % #3 = node name prefix
  % Root
  \node[box] (#3root) at (#1+1.5, #2+1.0) {};
  
  % Level 1 children
  \node[box] (#3l1) at (#1+0.8, #2+0.5) {};
  \node[box] (#3r1) at (#1+2.2, #2+0.5) {};
  \draw (#3root) -- (#3l1);
  \draw (#3root) -- (#3r1);
  
  % Level 2 children - left subtree
  \node[box] (#3ll1) at (#1+0.2, #2+0.0) {};
  \node[box] (#3lr1) at (#1+1.3, #2+0.0) {};  % moved left slightly
  \draw (#3l1) -- (#3ll1);
  \draw (#3l1) -- (#3lr1);
  
  % Level 2 children - right subtree  
  \node[box] (#3rl1) at (#1+1.7, #2+0.0) {};  % moved right slightly
  \node[box] (#3rr1) at (#1+2.8, #2+0.0) {};
  \draw (#3r1) -- (#3rl1);
  \draw (#3r1) -- (#3rr1);
  
  % Connection points at bottom of each leaf
  \coordinate (#3ll1-bot) at (#1+0.2, #2-0.05);
  \coordinate (#3lr1-bot) at (#1+1.3, #2-0.05);  % updated position
  \coordinate (#3rl1-bot) at (#1+1.7, #2-0.05);  % updated position
  \coordinate (#3rr1-bot) at (#1+2.8, #2-0.05);
  % Connection points at top of each leaf
  \coordinate (#3ll1-top) at (#1+0.2, #2+0.1);
  \coordinate (#3lr1-top) at (#1+1.3, #2+0.1);   % updated position
  \coordinate (#3rl1-top) at (#1+1.7, #2+0.1);   % updated position
  \coordinate (#3rr1-top) at (#1+2.8, #2+0.1);
}

% Usage: Two trees side by side
\begin{tikzpicture}[baseline=-0.5ex,
  box/.style={draw, fill=blue!15, minimum size=2.5mm, inner sep=0pt},
  dotline/.style={densely dotted, line width=0.6pt}
]
  % Top tree
  \binarytree{0}{0.3}{L}
  
  % Bottom tree (rotated 180 degrees)
  \begin{scope}[rotate=180, shift={(-3,0.2)}]
    \binarytree{0}{0}{R}
  \end{scope}
  
  % Connect corresponding leaves
% Connect corresponding leaves with spacing
  \draw[shorten >=2pt, shorten <=2pt] (Lll1-bot) -- (Rrr1-bot);
  \draw[shorten >=2pt, shorten <=2pt] (Llr1-bot) -- (Rrl1-bot);
  \draw[shorten >=2pt, shorten <=2pt] (Lrl1-bot) -- (Rlr1-bot);
  \draw[shorten >=2pt, shorten <=2pt] (Lrr1-bot) -- (Rll1-bot);
\end{tikzpicture}
    \caption{}
    \label{fig:tucker-norm}
  \end{subfigure}
  \label{fig:tn-norm-const}
  \caption{Normalization constant of various PTNs. (a) MPS, (b) CP, and (c) Tensor Tree}
\end{figure}

% \begin{figure}[tb]
%   \centering
%     \centering
%     \input{diagrams/mps}
%     \caption{}
% \end{figure}

\subsection{Probabilistic Modeling with $\mathrm{MPS}_\sigma$}
\label{sec:method--ptn-mps-sigma}
We now introduce the $\mathrm{MPS}_\sigma$ model, which enforces positivity of the underlying tensor by enforcing positivity on each of the cores independently. In other words,
\begin{equation}
    \label{eq:mps-sigma}
    p(y_1, \ldots, y_N) = \frac{
\Psi_\sigma (\vec y)
    }{Z}, \quad \Psi_\sigma(\vec y) = \sigma(\tensor{G}^{(1)}) [y_1] \cdots \sigma(\tensor{G}^{(N)})[y_N]
\end{equation}
where $\sigma: \mathbb{R} \to \mathbb{R}_{\geq 0}$ is a non-negative function, applied point-wise to tensor entries. The normalization constant in~\Eqref{eq:mps-sigma} can be computed efficiently and reduces to a sequence of matrix multiplications (see Appendix~\ref{app:mps-sigma-normalization}).

\begin{table}[b]							
\centering			
\caption{Trade-offs between different combinations of models and optimization routines (ME = Memory Efficient; PLL = Parallelizable; AD = compatible with Automatic Differentiation). Notably, LSF (ours) outperforms SGD and DMRG in terms of stability and computational intensity, respectively.
% We exclude combinations that do not have a clear definition 
% (e.g., DMRG with $\mathrm{MPS}_\sigma$, is not included as it would include enforcing positivity after the decomposition operation, thereby corrupting the parameter update, see~\ref{app:dmrg-mps-sigma} for more details).
}
\label{tab:tradeoffs-simple}
% \begin{tabular}{lllccccc}
% \toprule
% Optimization & Model & Canonical & Adaptive & Fast & ME & PLL & Stable \\
% \midrule
% 2-DMRG & $\mathrm{MPS}_\mathrm{BM}$ & Yes & \checkmark & x & x & x & \checkmark \\
% 2-DMRG & $\mathrm{MPS}_\mathrm{BM}$ & No & \checkmark & x & x & \checkmark & \checkmark \\
% 2-DMRG & $\mathrm{MPS}_\sigma$ & Yes & \checkmark & x & x & x & x \\
% 2-DMRG & $\mathrm{MPS}_\sigma$ & No & \checkmark & x & x & x & x \\
% 1-DMRG & $\mathrm{MPS}_\mathrm{BM}$ & Yes & x & x & x & \checkmark & \checkmark \\
% 1-DMRG & $\mathrm{MPS}_\sigma$ & Yes & x & x & x & x & x \\
% 1-DMRG & $\mathrm{MPS}_\sigma$ & No & x & x & x & x & x \\
% GD & $\mathrm{MPS}_\mathrm{BM}$ & No & x & \checkmark & \checkmark & \checkmark & x \\
% GD & $\mathrm{MPS}_\sigma$ & No & x & \checkmark & \checkmark & \checkmark & x \\
% LSF & $\mathrm{MPS}_\mathrm{BM}$ & No & x & \checkmark & \checkmark & \checkmark & \checkmark \\
% LSF & $\mathrm{MPS}_\sigma$ & No & x & \checkmark & \checkmark & \checkmark & \checkmark \\
% \bottomrule
% \end{tabular}
\begin{tabular}{llllcccc}
\toprule
Optimization & Model & Adaptive & Fast & ME & PLL & AD & Stable \\
\midrule
DMRG \citep{han2018unsupervised} & $\mathrm{MPS}_\mathrm{BM}$ & \cmark & \xmark & \xmark & \xmark & \xmark & \cmark \\
% 1-DMRG & $\mathrm{MPS}_\mathrm{BM}$ & x & x & x & \checkmark & \checkmark \\
SGD \citep{glasser2019expressive} & $\mathrm{MPS}_\mathrm{BM}$ & \xmark & \cmark & \cmark & \cmark & \cmark & \xmark \\
SGD \citep{glasser2019expressive} & $\mathrm{MPS}_\sigma$ & \xmark & \cmark & \cmark & \cmark & \cmark & \xmark \\
LSF (ours) & $\mathrm{MPS}_\mathrm{BM}$ & \xmark & \cmark & \cmark & \cmark & \cmark & \cmark \\
LSF (ours) & $\mathrm{MPS}_\sigma$ & \xmark & \cmark & \cmark & \cmark & \cmark & \cmark \\
\bottomrule
\end{tabular}					
\end{table}		

\subsection{Learning $\mathrm{MPS}_\mathrm{BM}$ and $\mathrm{MPS}_\sigma$}
\label{sec:method--learning-mps-params}

\begin{table}[tb]
\begin{minipage}[tb]{0.38\linewidth}
\centering
    \footnotesize
    \caption{
    Asymptotic time and space complexities of DMRG vs SGD.
    % Comparison of asymptotic time and space complexities between DMRG and SGD. DMRG scales quadratically with bond dimension $D$ while SGD maintains linear scaling
    }
    \label{tab:complexity}
    % \begin{tabular}{l l}
    %     \toprule
    %     \textbf{Complexity} & \textbf{Expression} \\
    %     \midrule
    %     \multicolumn{2}{l}{\textbf{Time}} \\
    %     \quad $\mathrm{MPS}_{\mathrm{BM}}$   & $\mathcal{O}(\mathrm{NR}^3\mathrm{V} + \mathrm{NRV}^2)$ \\
    %     \quad $\mathrm{MPS}_{\sigma}$        & $\mathcal{O}(\mathrm{NR}^2\mathrm{V})$ \\
    %     \midrule
    %     \multicolumn{2}{l}{\textbf{Space}} \\
    %     \quad $\mathrm{MPS}_{\mathrm{BM}}$   & $\mathcal{O}(\mathrm{NVR^2} + \mathrm{V^2R^2})$ \\
    %     \quad $\mathrm{MPS}_{\sigma}$        & $\mathcal{O}(\mathrm{NVR^2})$ \\
    %     \bottomrule
    % \end{tabular}
    \begin{tabular}{@{\hspace{2pt}}l@{\hspace{6pt}}l@{\hspace{2pt}}}
    \toprule
    \textbf{Complexity} & \textbf{Expression} \\
    \midrule
    \multicolumn{2}{@{\hspace{2pt}}l@{\hspace{2pt}}}{\textbf{Time}} \\
    \quad $\mathrm{MPS}_{\mathrm{BM+DMRG}}$   & $\mathcal{O}(\mathrm{NR^3}\mathrm{D} + \mathrm{NR^2D^2})$ \\
    \quad $\mathrm{MPS}_{\sigma+\mathrm{SGD}}$        & $\mathcal{O}(\mathrm{NR}^3 + \mathrm{ND})$ \\
    \midrule
    \multicolumn{2}{@{\hspace{2pt}}l@{\hspace{2pt}}}{\textbf{Space}} \\
    \quad $\mathrm{MPS}_{\mathrm{BM+DMRG}}$   & $\mathcal{O}(\mathrm{NDR^2} + \mathrm{D^2R^2})$ \\
    \quad $\mathrm{MPS}_{\sigma+\mathrm{SGD}}$        & $\mathcal{O}(\mathrm{NDR^2})$ \\
    \bottomrule
\end{tabular}
\end{minipage}\hfill
\begin{minipage}[tb]{0.55\linewidth}
\centering
    \includegraphics[width=\linewidth]{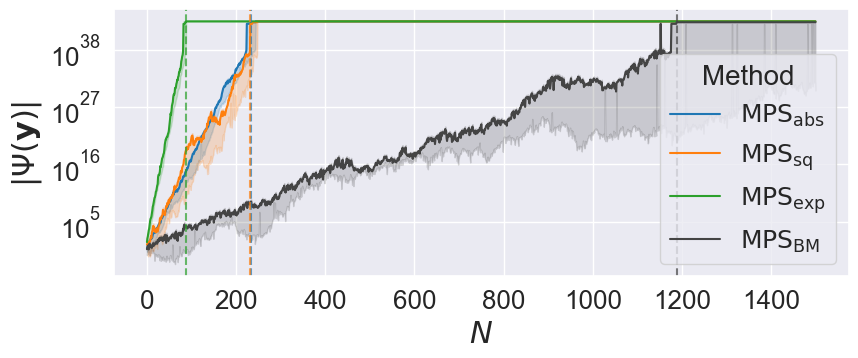}
    \captionof{figure}{
Magnitude of numerator terms in  and Equations~\ref{eq:mps-born} and \ref{eq:mps-sigma} as $N$ is increased for MPS-based models.
% Notably, the computation of the unnormalized probability is stable for at-least an order of magnitude more for $\mathrm{MPS}_\mathrm{BM}$. This underscores the need for logarithmic scale factors during training of $\mathrm{MPS}_\sigma$ compared with $\mathrm{MPS}_\mathrm{BM}$.
    }
    \label{fig:psi}
\end{minipage}
\end{table}

Previous methods have shown that DMRG can be used to learn the parameters of $\mathrm{MPS}_\mathrm{BM}$~\citep{han2018unsupervised}, and SGD can be used to learn the parameters of both $\mathrm{MPS}_\mathrm{BM}$ and $\mathrm{MPS}_\sigma$~\citep{glasser2019expressive} by minimizing the negative log likelihood. However, both approaches have shortcomings that we summarize in Table~\ref{tab:tradeoffs-simple}. While DMRG is numerically stable, it is computationally intensive (see Table~\ref{tab:complexity}) and not fully compatible with automatic differentiation, thereby making it difficult to integrate into machine learning frameworks (see  Section~\ref{sec:method--automatic-differentiation}). We also note that the combination of DMRG and $\mathrm{MPS}_\sigma$ is not well defined.
% We also  note that the combination of DMRG and $\mathrm{MPS}_\sigma$ is omitted from Table~\ref{tab:tradeoffs-simple} as it is not well defined. 
This is because naively applying a non-linearity after the decomposition would corrupt the parameter update (see Appendix~\ref{app:dmrg-mps-sigma}). This incompatibility further restricts the applicability of DMRG for training PTNs. Given the downsides of using DMRG, we revisit using vanilla SGD as in~\cite{glasser2019expressive}.

% While Sections~\ref{sec:intro}, \ref{sec:method--automatic-differentiation} \& \ref{sec:method--computational-complexity} highlight the downsides of using $\mathrm{MPS}_\mathrm{BM}$ in combination with the DMRG algorithm for training, in this section we focus on the issues with using vanilla stochastic gradient descent to train both models. Our discussion is summarized in Table~\ref{tab:tradeoffs-simple}. We also note that the combination of DMRG and $\mathrm{MPS}_\sigma$ is omitted as it is not well defined. This is because naively applying a non-linearity after the decomposition would corrupt the parameter update (See Appendix~\ref{app:dmrg-mps-sigma}).

\textbf{Why can we not use vanilla SGD to train $\mathrm{MPS}_\sigma$ models?}

The $\mathrm{MPS}_\sigma$ model enforces positivity of the underlying tensor by applying a point-wise positivity function to each of the MPS cores as in~\Eqref{eq:mps-sigma}. However, this constraint causes both the numerator and denominator in \Eqref{eq:mps-sigma} to grow rapidly with the number of cores as shown in Figure~\ref{fig:psi}. 
% In comparison  $\mathrm{MPS}_\mathrm{BM}$ does not suffer as drastic of a growth in magnitude since negative terms exist in its contraction that cancel out some of the positive terms. 
We characterize this growth in Theorem~\ref{thm:mps-sigma-stability}, showing that the expected value of both the numerator and denominator grow \emph{exponentially} with the rank dimension (proof in Appendix~\ref{app:proofs})

\begin{restatable}{theorem}{thmMpsSigmaInstability}
% \begin{theorem}
\label{thm:mps-sigma-stability}
    Let the elements of the tensor $\tensor{g}^{(i)} \in \mathbb{R}^{R_i\times D\times R_{i+1}}$ be i.i.d. random variables drawn from a zero-mean gaussian distribution with unit variance, $R_1 = R_N = 1$ and $R_i = R \quad\forall i \neq 1,N$. Let $\vec y\in \mathcal{Y},\; \mathcal{Y} = \mathcal{Y}_i \times \cdots \times \mathcal{Y}_N\;$  and 
    \begin{align}
        \Psi_\sigma(\vec y) &= \sigma(\tensor{g}^{(1)}[y_1]) \cdots \sigma(\tensor{g}^{(N)}[y_N]), \quad
        Z_\sigma = \dot{\tensor{g}}^{(1)} \cdots \dot{\tensor{g}}^{(N)},
    \end{align}
where $\dot{\tensor{g}}_{ij} \triangleq \sum_k \sigma(\tensor{g}_{ikj})\;$ and $\;\sigma : \mathbb{R} \to \mathbb{R}_{\geq 0}$ is a point-wise non-negative mapping  s.t. 
\[
 \forall x \in \mathbb{R}_{>0}, \; \exists \, \epsilon_x > 0 \;\; \text{s.t.} \;\; \sigma(x) > \epsilon_x.
\]
Then, $\E[\Psi_\sigma(\vec y)] \geq \epsilon R^N\,$ and $\;\E[Z_\sigma] \geq \epsilon R^N D^N$ for some $\epsilon >0$
% \end{theorem}
\end{restatable}

\textbf{Why can we not use vanilla SGD to train $\mathrm{MPS}_\mathrm{BM}$ models?}

In contrast to $\mathrm{MPS}_\sigma$, $\mathrm{MPS}_\mathrm{BM}$ does not enforce positivity on cores and does not suffer as drastic a growth in magnitude, since cancellation can occur between negative and positive terms in the tensor contraction as shown in Figure~\ref{fig:psi}. However, while the expected value of $\Psi_\mathrm{BM}(\vec{y})$ in Equation~\ref{eq:mps-born} is zero (with respect to cores $\tensor{g}$), its \emph{variance} grows exponentially with the rank dimension $R$, as shown in~Theorem~\ref{thm:mps-bm-stability} (proof in Appendix~\ref{app:proofs})

\begin{restatable}{theorem}{thmMpsBMInstability}
% \begin{theorem}
\label{thm:mps-bm-stability}
    Let the elements of the tensor $\tensor{g}^{(i)} \in \mathbb{R}^{R_i\times D\times R_{i+1}}$ be i.i.d. random variables drawn from a zero-mean gaussian distribution with unit variance, $R_1 = R_N = 1$ and $R_i = R \quad\forall i \neq 1,N$. Let $\vec y\in \mathcal{Y},\; \mathcal{Y} = \mathcal{Y}_i \times \cdots \times \mathcal{Y}_N\;$  and
    % \begin{align}
    $
        \Psi_\mathrm{BM}(\vec y) = \tensor{g}^{(1)}[y_1] \cdots \tensor{g}^{(N)}[y_N].
    $
% where $\vec y\in [D]^N$.
Then,  $\mathbb{E}[\Psi_\mathrm{BM} (\vec y)] = 0\;$ and $\;\mathbb{E}[\Psi_\mathrm{BM}(\vec y)^2] \geq \epsilon R^N$ for some $\epsilon >0$
% \end{theorem}
\end{restatable}

Overall, the instability of using vanilla SGD for training PTNs severly limits it's applicability to real world datasets. For instance, in~\cite{glasser2019expressive} only datasets consisting of a maximum of 22 variables were considered.

% Glasser dataset num feats:
% biofam: 16
% lymph: 19
% SPECT: 22
% Voting: 16
% Tumor: 17
% Solar Flare: 10

\subsection{SGD with Logarithmic Scale Factors}
\label{sec:method--lsf}
\begin{algorithm}[tb]
\caption{LSF (Stochastic Gradient Descent with Logarithmic Scale Factors)}
\label{alg:lsf}
\begin{algorithmic}[1]
\Require Data $Y \in \N^{N_\mathrm{samples} \times D}$, parameters $g_i \in \R^{R_i \times D_i \times R_{i+1}}$
\Ensure Updated parameters $\{g_i^\star\}_{i=1}^N$
% \State Initialize $g_i \in \R^{R_i \times D_i \times R_{i+1}}$
\For{$i = 1 \dots N_\mathrm{samples}$}
    \State $\tilde{p}_1 \gets \mathbf{1}_1$, $z_1 \gets \mathbf{1}_1$ \Comment{One dimensional vector}
    \State $\gamma_1^{(p)} \gets 1$, $\gamma_1^{(z)} \gets 1$
    \State $\tilde{p}_n \gets \frac{1}{\gamma_n^{(p)}}\tilde{p}_{n-1}^T \sigma(\tensor{g}^{(n-1)}[Y_{in}])$
    \State $\gamma_n^{(z)} \gets \max (\tilde{p}_{n-1}^T \sigma(\tensor{g}^{(n-1)}[Y_{in}]))$
    \State $z_n \gets \frac{1}{\gamma_n^{(z)}} z_{n-1}^T \dot{\tensor{g}}^{(n-1)}[Y_{in}]$
    \State $l \gets (\log \sum_{n} z_{N+1} - \tilde{p}_{N+1}) + \sum_j \log \gamma^{(z)}_{j} - \log \gamma^{(p)}_{j}$
    \State $\theta \gets \theta - \alpha \nabla_\theta l$
\EndFor
\State \Return $\theta$
\end{algorithmic}
\end{algorithm}

This section introduces our numerically stable method for computing the negative log-likelihood of MPS-based probabilistic tensor networks. Given a dataset of $K$ independent and identically distributed (i.i.d.)
observations
$
\mathcal{D} = \left\{ \big(y^{(k)}_1, \ldots, y^{(k)}_N\big) \right\}_{k=1}^K,
$
we learn the model parameters $\tensor{g}^{(i)}$ of an MPS parameterized distribution using maximum likelihood estimation.
Namely, we seek to minimize the empirical negative log-likelihood
\begin{equation}
\label{eq:nll-general}
\ell(\tensor{g}) 
= - \frac{1}{K} \sum_{k=1}^K 
   \log p \big(y^{(i)}_1, \ldots, y^{(i)}_N \big).
\end{equation}
where $p$ denotes the MPS-parameterized distribution given by \Eqref{eq:mps-sigma}. Computing the sequence of matrix multiplications in~\Eqref{eq:mps-sigma} leads to numerical overflow beyond small sequence lengths ($100$ or more), as shown in Figure~\ref{fig:stability-iterations}. We make the key observation that since we are ultimately concerned with computing log probabilities, we can factor out \emph{logarithms of scale factors} in order to stabilize the computation. Thus, we can compute the loss in a numerically stable fashion as follows:
\begin{equation}
\label{eq:nll-stable}
\ell(\tensor{g}) = \log \tilde{Z} - \log \tilde{p}(y_1, \ldots, y_N) + \sum_n \log \gamma^{(z)}_n - \log \gamma^{(p)}_n
\end{equation}
where~~$\tilde{\tensor{g}}^{(n)} = \frac{1}{\gamma_n^{(p)}}  \tilde{\tensor{g}}^{(1)}[y_1] \cdots \tilde{\tensor{g}}^{(n-1)}[y_{n-1}] \,\sigma(\tensor{g}^{(n)}[y_{n}])$ and $\gamma_i^{(p)}, \gamma_i^{(z)}$ are scale factors enabling the stable computation of $\tilde p$ and $\tilde Z$, respectively, as
\begin{align*}
\tilde{p}(y_1, \ldots, y_N) &= \tilde{\tensor{g}}^{(1)}[y_1] \cdots \tilde{\tensor{g}}^{(N)}[y_N], \quad
\gamma_n^{(p)} = \left\| \tilde{\tensor{g}}^{(1)}[y_1] \cdots \tilde{\tensor{g}}^{(n-1)}[y_{n-1}] \sigma(\tensor{g}^{(n)}[y_{n}]) \right\|.
\end{align*}
The same progression can be applied for the normalization constant, thus stabilizing the computation of the loss function in~\Eqref{eq:nll-stable}:
\begin{align*}
\dtilde{Z} &= \dtilde{\tensor{g}}^{(1)} \cdots \dtilde{\tensor{g}}^{(N)}, \quad 
\dtilde{\tensor{g}}^{(n)} = \frac{1}{\gamma_n^{(p)}}  \dtilde{\tensor{g}}^{(1)} \cdots \dtilde{\tensor{g}}^{(n-1)}
\dot{\tensor{g}}^{(n)} \\
\gamma_n^{(z)} &= \left\| \dtilde{\tensor{g}}^{(1)}[y_1] \cdots \dtilde{\tensor{g}}^{(n-1)}[y_{n-1}] \dot{\tensor{g}}^{(n)}[y_{n}] \right\| \quad
\dot{\tensor{g}}^{(i)} = \sum_{y_i \in \mathcal{Y}_i} \sigma(\tensor{g}^{(i)}[y_i^\prime]).
\end{align*}

The overall procedure is provided in Algorithm~\ref{alg:lsf}.

\subsection{Compatibility with Automatic Differentiation}
% \subsection{Comparing LSF with DMRG}
\label{sec:method--automatic-differentiation}

The method proposed in Section~\ref{sec:method--lsf} enables the stable computation of the negative log-likelihood using~\Eqref{eq:nll-stable}, thus end-to-end learning of PTN model parameters can be performed. In contrast, the DMRG algorithm uses the negative log likelihood to compute updates with respect to a fourth-order tensor that is subsequently decomposed using SVD. This decomposition serves two purposes: (i) it enables adaptive learning of bond dimensions and (ii) it maintains isometry of cores, which in turn stabilizes the computation of the loss. 

A single step of the DMRG algorithm used in~\cite{han2018unsupervised} is depicted in Figure~\ref{fig:dmrg}. First, the neighboring cores $\tensor{g}^{(1)}$ and $\tensor{g}^{(2)}$ are merged. Second, the loss function (negative log likelihood) is computed. Third, the gradient of the loss function is computed with respect to the fourth-order tensor (this can be done using automatic differentiation) and used to update the fourth order tensor using gradient descent. Lastly, the updated the fourth order tensor is decomposed using SVD and singular vectors are copied into the model parameters $\tensor{g}^{(1)}, \tensor{g}^{(2)}$. Crucially, the last step is not an ancestor of the loss function computation, thus model parameters cannot be updated end-to-end using automatic differentiation.

 % then singular vectors are \emph{copied} into cores $\tensor{g}^{(1)}$ and  $\tensor{g}^{(2)}$.}

\begin{figure}[tb]
  \centering
  \begin{minipage}[t]{0.63\textwidth}
    \centering
    \scalebox{0.8}{\input{diagrams/dmrg}}
    \caption{Illustration of a single update step using the DMRG two site update algorithm used in~\cite{han2018unsupervised}. (1) Cores $\tensor{g}^{(1)}$ and $\tensor{g}^{(2)}$ are merged, (2) the loss is computed with respect to the merged fourth order tensor, (3) the gradient is computed and used to update the fourth order tensor using automatic differentiation and (4) the fourth order tensor is decomposed using SVD, then singular vectors are \emph{copied} into cores $\tensor{g}^{(1)}$ and  $\tensor{g}^{(2)}$.}
    \label{fig:dmrg}
  \end{minipage}
  \hfill
  \begin{minipage}[t]{0.34\textwidth}
    \centering
    \includegraphics[width=\linewidth]{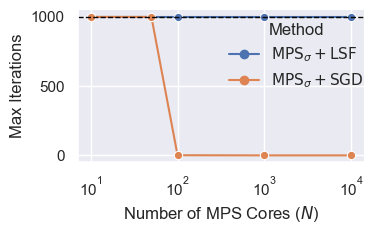}
    \caption{Maximum number of iterations reached during training using vanilla stochastic gradient descent $\mathrm{MPS}_{\sigma + \mathrm{SGD}}$ vs. stochastic gradient descent with logarithmic scale factors $\mathrm{MPS}_{\sigma + \mathrm{LSF}}$ (ours).}
    % , the former overflows at $N=100$.}
    \label{fig:stability-iterations}
  \end{minipage}
\end{figure}

\subsection{Sampling from MPS-based models}
\label{sec:method--sampling}
Sampling from $\mathrm{MPS}$-based models can de done efficiently and reduces to performing a sequence of matrix multiplications. Conditional sampling can also be performed efficiently, due to the tractable computation of marginals. For instance, in order to sample in an auto-regressive fashion, we can compute the conditional distribution for the $n^\text{th}$ position given the past as follows:
\begin{align*}
p(y_n \mid y_1, \ldots, y_{n-1})
    &= \frac{p(y_1, \ldots, y_n)}{p(y_1, \ldots, y_{n-1})} = \frac{
    \tensor{g}^{(1)}[y_1] \cdots \tensor{g}^{(n)}[y_n]
    \dot{\tensor{g}}^{(n+1)}[y_{n+1}] \cdots 
       \dot{\tensor{g}}^{(N)}[y_N]
       }{
    Z
       },
\end{align*}
where $\dot{\tensor{g}}_{ij} = \sum_k \tensor{g}_{ikj}$. Notably, with MPS-based models we can sample in any order and from any marginal distribution (see Appendix~\ref{app:mps-sampling}). 

\section{Experiments}
In this section we compare our method LSF with both vanilla SGD and DMRG for training different MPS-based probabilistic tensor networks. Section~\ref{sec:experiments--stability} compares the stability of LSF vs. SGD.  Section \ref{sec:experiments--complexity} compares latency and memory requirements of LSF vs. DMRG for varying MPS model dimensions. Section~\ref{sec:experiments--ucla} compares the performance of LSF vs. SGD on various density estimation benchmarks and Section~\ref{sec:experiments--mnist} compares the performance of LSF vs. DMRG on MNIST.

\begin{figure*}[tb]
    \centering
    % Left: Latency
    \begin{subfigure}[t]{0.68\textwidth}
        \centering
        \includegraphics[width=\linewidth]{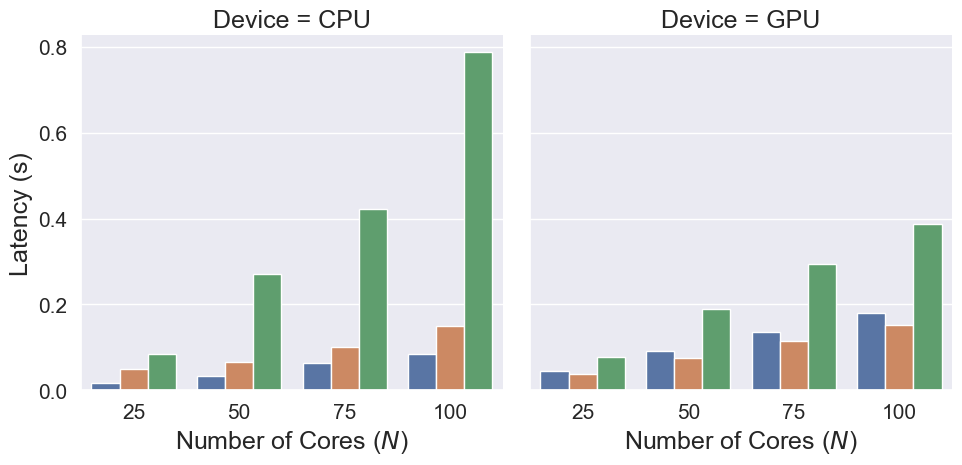}
        \caption{}
        % \caption{Latency vs. Number of cores ($N$) for $\mathrm{MPS}_\sigma$ and $\mathrm{MPS}_{\mathrm{BM}}$ (CPU/GPU).}
        \label{fig:latency}
    \end{subfigure}
    \begin{subfigure}[t]{0.31\textwidth}
        \centering
        \includegraphics[width=\linewidth]{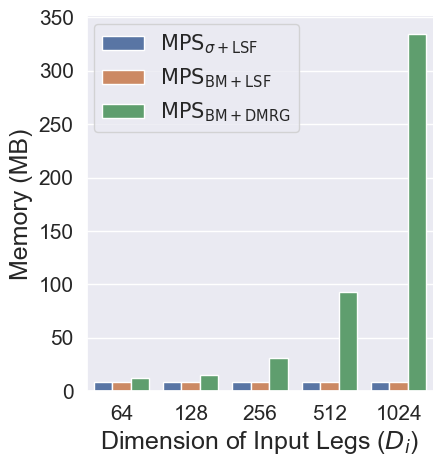}
        \caption{}
        % \caption{Peak Memory vs. Dimension of free leg ($D$).}
        \label{fig:peak-memory}
    \end{subfigure}
    \caption{(a) Latency of single update to all model parameters using LSF and DMRG, on both CPU and GPU for various number of cores ($N$).
    (b) Peak memory encountered during a single update to all model parameters using LSF and DMRG for various free leg dimensions.}
    \label{fig:latency-memory-combined}
\end{figure*}

\begin{minipage}[t]{0.52\textwidth}
{\centering
\captionof{table}{Average test negative log-likelihood achieved by different tensor network models trained on MNIST. Notably, $\mathrm{MPS}_{\sigma\mathrm{+LSF}}$ achieves performance close to $\mathrm{MPS}_{\mathrm{BM+DMRG}}$ while being $10\times$ faster.}
\label{tab:mnist-results}
\begin{tabular}{l c c}
\toprule
\textbf{Model} & \textbf{NLL} & \textbf{Latency (s)}  \\
\midrule
$\mathrm{PixelCNN}$ & 0.104 & -  \\
\midrule
$\mathrm{MPS}_\mathrm{BM + DMRG}$ & 0.129 & 1.20  \\
$\mathrm{MPS}_{\mathrm{exp}\mathrm{+ LSF}}$ (ours) & 0.136 & 0.11 \\
$\mathrm{MPS}_{\mathrm{abs}\mathrm{+ LSF}}$ (ours) & 0.140 & 0.12 \\
$\mathrm{MPS}_{\mathrm{sig}\mathrm{+ LSF}}$ (ours) & 0.148 & 0.98 \\
$\mathrm{MPS}_\mathrm{BM+LSF}$ (ours) & 0.168 & 0.12  \\
\bottomrule
\end{tabular}
}
% \subsection{Performance of LSF vs. SGD on Density Estimation Benchmarks}
% \label{sec:experiments--ucla}
% \noindent
% This section compares the generalization performance of MPS-based models trained with SGD~\cite{glasser2019expressive} vs. LSF on 20 density estimation benchmarks~\cite{lowd2010learning, vanhaaren2012markov}. We also compare against EiNet, a state-of-the-art probabilistic circuit with tractable marginals~\citep{peharz2020einet}. Specifically, we train $\mathrm{MPS}_\sigma$ for 50 epochs with batch size 32, bond dimension of 32, learning rate of 5e-3 and select the exponential function for positivity. Table~\ref{tab:ucla-results} reports the best test set performance. The~~\xmark~~symbol indicates numerical overflow before completing the first epoch. Notably, our method achieves comparable performance with EiNet, meanwhile the approach in~\cite{glasser2019expressive} results in numerical overflow on most datasets.

\subsection{Comparing the stability of LSF vs. SGD}
\label{sec:experiments--stability}

This section analyzes the numerical stability of training MPS-based models using stochastic gradient descent with logarithmic scale factors (LSF) vs. vanilla stochastic gradient descent (SGD) as in~\citet{glasser2019expressive}. In Figure~\ref{fig:stability-iterations} we train $\mathrm{MPS}_\sigma$ for up to 1k iterations while varying the number of cores. For smaller systems (approximately $N\approx 50$), SGD successfully updates the MPS cores. However, for systems exceeding $N=100$, numerical overflow prevents more than a single iteration to be performed. In contrast, LSF enables training for the maximum number of iterations even with 10k MPS cores. 
\end{minipage}
\hfill
\begin{minipage}[t]{0.46\textwidth}
\centering
\captionof{table}{Average test negative log-likelihood for LSF compared to SGD~\citep{glasser2019expressive} and EiNet (EN)~\citep{peharz2020einet} methods. The † symbol indicates numerical overflow occurred before training completion, while ~\xmark~ indicates numerical overflow before completing a single epoch.}
\label{tab:ucla-results}
\begin{tabular}{lrccc}
\toprule
Dataset & $N$ & EN & SGD & LSF \\
\midrule
% nltcs & 16 & 0.38 & 0.38 & 0.38 \\
% msnbc & 17 & 0.35 & 0.36 & 0.36 \\
% kdd-2k & 64 & 0.03 & 0.33^\dagger & 0.03 \\
% plants & 69 & \textbf{0.20} & 0.37^\dagger & 0.24 \\
% jester & 100 & 0.53 & \xmark & 0.54 \\
% audio & 100 & 0.40 & \xmark & 0.42 \\
% netflix & 100 & 0.57 & \xmark & \textbf{0.59} \\
% accidents & 111 & 0.34 & \xmark & 0.35 \\
% retail & 135 & 0.08 & \xmark & 0.08 \\
% pbstar & 163 & 0.24 & \xmark & 0.23 \\
% dna & 180 & 0.54 & \xmark & \textbf{0.44} \\
% kosarek & 190 & 0.06 & \xmark & 0.06 \\
% msweb & 294 & 0.04 & \xmark & 0.04 \\
% book & 500 & 0.07 & \xmark & 0.07 \\
% movie & 500 & 0.11 & \xmark & 0.12 \\
% web-kb & 839 & 0.19 & \xmark & 0.20 \\
% r52 & 889 & 0.10 & \xmark & 0.11 \\
% 20ng & 910  & 0.17 & \xmark & 0.18 \\
% bbc & 1058 & 0.25 & \xmark & 0.26 \\
% ad & 1556 & 0.04 & \xmark & 0.04 \\
nltcs & 16 & 0.38 & 0.38 & 0.38 \\
msnbc & 17 & 0.35 & 0.36 & 0.36 \\
kdd-2k & 64 & 0.03 & $0.33^\dagger$ & 0.03 \\
plants & 69 & \textbf{0.20} & $0.37^\dagger$ & 0.24 \\
jester & 100 & 0.53 & \xmark & 0.54 \\
audio & 100 & 0.40 & \xmark & 0.42 \\
netflix & 100 & 0.57 & \xmark & 0.59 \\
accidents & 111 & 0.34 & \xmark & 0.35 \\
retail & 135 & 0.08 & \xmark & 0.08 \\
pbstar & 163 & 0.24 & \xmark & 0.23 \\
dna & 180 & 0.54 & \xmark & \textbf{0.44} \\
kosarek & 190 & 0.06 & \xmark & 0.06 \\
msweb & 294 & 0.04 & \xmark & 0.04 \\
book & 500 & 0.07 & \xmark & 0.07 \\
movie & 500 & 0.11 & \xmark & 0.12 \\
web-kb & 839 & 0.19 & \xmark & 0.20 \\
r52 & 889 & 0.10 & \xmark & 0.11 \\
20ng & 910  & 0.17 & \xmark & 0.18 \\
bbc & 1058 & 0.25 & \xmark & 0.26 \\
ad & 1556 & 0.04 & \xmark & 0.04 \\
\bottomrule
\end{tabular}
\end{minipage}

\subsection{Latency and Memory usage of LSF vs. DMRG}
\label{sec:experiments--complexity}
We analyze the latency and peak memory usage of LSF compared with DMRG in Figure~\ref{fig:latency-memory-combined} on both CPU and GPU. We set the batch size, ranks $R_i$ and free legs $D_i$ to 32, 8 and 2 respectively. Meanwhile, we vary the number of cores $N$. As shown in Figure~\ref{fig:latency}, $\mathrm{MPS}_{\sigma\mathrm{+LSF}}$ and $\mathrm{MPS}_{\mathrm{BM+LSF}}$ achieve drastic speedups over $\mathrm{MPS}_\mathrm{BM+DMRG}$~\citep{han2018unsupervised} as the number of cores $N$ is increased. For instance, with $N=100$, $\mathrm{MPS}_\mathrm{BM+DMRG}$ requires ~\textbf{0.8} seconds to perform one update to all model parameters; meanwhile, $\mathrm{MPS}_\mathrm{\sigma +LSF}$ requires \textbf{0.09} seconds, leading to approximately one order of magnitude speedup. 

We compare the peak memory usage of both methods in Figure~\ref{fig:peak-memory} at various input dimensions. As illustrated in Figure~\ref{fig:crown-jewel}, $\mathrm{MPS}_{\mathrm{BM+DMRG}}$ requires the materialization of fourth-order tensors during training. We show in Figure~\ref{fig:peak-memory} that this quickly leads to extreme  memory consumption. For instance, at $D_i=1024$ $\mathrm{MPS}_\mathrm{BM+DMRG}$ requires \textbf{334 MB} compared with only \textbf{8 MB} for $\mathrm{MPS}_{\sigma+\mathrm{LSF}}$.

% Original text
\subsection{Performance of LSF vs. SGD on Density Estimation Benchmarks}
\label{sec:experiments--ucla}
This section compares the generalization performance of MPS-based models trained using SGD~\citep{glasser2019expressive} vs. LSF on 20 density estimation benchmarks~\citep{lowd2010learning, vanhaaren2012markov}. We also compare against EiNet, a state-of-the-art probabilistic circuit with tractable marginals~\citep{peharz2020einet}. Specifically, we train $\mathrm{MPS}_{\sigma+\mathrm{LSF}}$ for 50 epochs with batch size 32, bond dimension of 32, learning rate of 5e-3 and select the exponential function for positivity. Table~\ref{tab:ucla-results} reports the best test set performance. The~~\xmark~~symbol indicates numerical overflow before completing the first epoch. Notably, our method achieves comparable performance with EiNet, meanwhile the approach in~\cite{glasser2019expressive} results in numerical overflow on most datasets. Specifically, \textbf{SGD fails entirely on all datasets with 100 random variables or more} and partially on datasets consisting of $\sim60$ random variables.

% In this section we compare the generalization ability of MPS-based models trained using DMRG vs. LSF on the task of learning to generate MNIST digits.  We use 60k and 10k samples for training and testing respectively. Each image is flattened then binarized resulting in a binary vector of length 784. MPS based models with $N=784, \; D_i=2 \; R=32$ are then trained using LSF. Notably, in Table~\ref{tab:mnist-results}  we show that $\mathrm{MPS}_{\sigma+\mathrm{LSF}}$ achieves performance close to $\mathrm{MPS}_\mathrm{BM+DMRG}$ while being $\sim10\times$ faster. We note that the memory improvements of LSF over DMRG in this experiment are marginal as the memory usage of DMRG grows exponentially in input dimension $D$ which is only equal to 2 in this experiment. We also compare against PixelCNN~\cite{van2016pixelrnn}, as it acheives state-of-the-art results in this task. While PixelCNN outperforms both MPS-based methods, it has intractable marginals and does not allow for sophisticated sampling.

\subsection{Comparing MNIST Generalization Performance of LSF vs. DMRG}
\label{sec:experiments--mnist}

This section compares the generalization performance of MPS-based models trained with DMRG vs. LSF on the task of learning to generate MNIST digits. We use 60,000 training samples and 10,000 test samples. Each image is flattened and binarized to produce a 784-dimensional binary vector. We then train MPS models using LSF, setting $N=784,\, D_i=2, \text{ and } R_i=32$. Table~\ref{tab:mnist-results} demonstrates that $\mathrm{MPS}_{\sigma+\mathrm{LSF}}$ achieves performance comparable to $\mathrm{MPS}_\mathrm{BM+DMRG}$ while providing approximately \textbf{$\mathbf{10\times}$ speedup}. The memory advantages of LSF over DMRG are negligible in this experiment as DMRG's memory usage scales quadratically with input dimensions $D_i$, which only equals two in this experiment. We also benchmark against PixelCNN, which achieves state-of-the-art performance on this task. Although PixelCNN outperforms both MPS approaches, it lacks tractable marginals, thus cannot be used for inference of complex queries.

Since LSF enables training a wider range of MPS-based models than DMRG, we experiment with various positivity enforcing functions. We find that the $\mathrm{MPS}_\sigma$ models generally outperforms $\mathrm{MPS}_\mathrm{BM}$ when using LSF, and that among $\mathrm{MPS}_\sigma$ models, using the exponential function often lead to the best performance.

\section{Conclusion}

Probabilistic Tensor Networks (PTNs) enable tractable inference over high-dimensional distributions, but face significant training challenges. Previous work has been limited to small-scale experiments ($<50$ variables) or relied on the computationally intensive DMRG algorithm for stable learning of a particular subset of PTNs. Beyond its computational cost, the reliance on DMRG presents a significant barrier to experimentation with PTNs, as DMRG implementations require non-trivial cache management for efficient batch processing, and do not leverage automatic differentiation for end-to-end model training~\citep{UnsupGenModbyMPS}. 

In this work, we addressed these limitations by introducing a stable method for the computation of the negative log-likelihood based on logarithmic scale factors. This approach enables larger scale training of PTNs, making them more practical for real-world applications. These advances also enable experimentation with PTNs using standard deep learning pipelines, while also opening exploration of the broad $\mathrm{MPS}_\sigma$ class of PTNs.


\bibliography{iclr2026_conference}
\bibliographystyle{iclr2026_conference}

\newpage
\appendix
\section{Appendix}
\label{app:crown-jewel-hps}

\subsection{Experimental details for Figure~\ref{fig:crown-jewel}d}
In this experiment we use the hyper-parameters listed in Table~\ref{tab:hyperparams}. The instability metric is computed using the following equation
$$
\mathrm{Instability} = \mathrm{Max~Iterations~Reached} - 10000 + 0.1,
$$
where the maximum number of iterations possible is 10k.

\begin{table}[h]
\centering
\caption{Hyper-parameters for experiments shown in~Figure~\ref{fig:crown-jewel}d.}
\label{tab:hyperparams}
\begin{tabular}{lccc}
\toprule 
\textbf{HP} & \textbf{Latency} & \textbf{Instability} & \textbf{Memory}  \\
\midrule
Batch Size       & 32   & 32    & 32   \\
Rank             & 2    & 2     & 2    \\
Input leg        & 2    & 2     & 1024 \\
Number of cores  & 100  & 100   & 5   \\
\bottomrule
\end{tabular}
\end{table}

\subsection{Computing the normalization constant of $\mathrm{MPS}_\sigma$}
\label{app:mps-sigma-normalization}
The normalization constant $Z$ of the $\mathrm{MPS}_\sigma$ class can be computed in time linear in $N$, following a similar algebraic simplification as shown for the Born Machine in~\Eqref{eq:mps-normalization-constant},
\begin{align}
    \label{eq:mps-sigma-normalization-constant}
Z 
&= \sum_{y_1^\prime, \ldots y_N^\prime \in \mathcal{Y}^N} 
    \sigma\left(\tensor{g}^{(1)}[y^\prime_1]\right) \cdots \sigma\left(\tensor{g}^{(N)}[y^\prime_N]\right) \\
&= \sum_{r_1, r_2, y^\prime_1} 
\sigma\left(\tensor{g}^{(1)}_{r_1, r_2}[y^\prime_1] \right)
\;\;\cdots 
\sum_{r_N, r_{N+1}, y^\prime_N} 
\sigma\left(\tensor{g}^{(N)}_{r_N, r_{N+1}}[y^\prime_N]\right).
\end{align}

\subsection{Using DMRG with $\mathrm{MPS}_\sigma$}
\label{app:dmrg-mps-sigma}

The combination of DMRG and $\mathrm{MPS}_\sigma$ is not well defined. As shown in Figure~\ref{fig:dmrg-repr}, the last step of the DMRG algorithm involves performing SVD in order to obtain an optimal low-rank decomposition of the matricization of tensor $\tilde{\tensor{g}}$, thereby solving
\begin{equation}
    \argmin_{\mat{g}^{(1)}, \mat{g}^{(2)}}  \left\| \tilde{\mat{g}} - \mat{g}^{(1)}\mat{g}^{(2)} \right\|,
\end{equation}
where $\mat{G}^{(1)} \in \R^{m\times r}$, $\mat{G}^{(2)} \in \R^{r\times n}$. However, in order for DMRG to apply to the $\mathrm{MPS}_\sigma$ class a different optimization problem must be solved, namely
\begin{equation}
    \argmin_{\mat{g}^{(1)}, \mat{g}^{(2)}}  \left\| \tilde{\mat{g}} - \sigma\left(\mat{g}^{(1)}\right)\sigma\left(\mat{g}^{(2)}\right) \right\|.
\end{equation}

\subsection{Backward sampling from an MPS-based distribution}
\label{app:mps-sampling}
As MPS-based models have tractable marginals, inference of more sophisticated queries is possible. For example, backward auto-regressive sampling can be performed using 
\begin{align*}
p(y_n | y_{n+1}, \ldots y_N)
    &= \frac{p(y_n, \ldots, y_N)}{p(y_{n+1}, \ldots, y_{N})} \\
    &= \frac{
    \dot{\tensor{g}}^{(1)}[y_1] \cdots \dot{\tensor{g}}^{(n-1)}[y_{n-1}]
    \tensor{g}^{(n)}[y_n]
    \tensor{g}^{(n+1)}[y_{n+1}] \cdots 
       \tensor{g}^{(N)}[y_N]
       }{
    Z
       }.
\end{align*}
In \cite{han2018unsupervised}, the authors provide examples of image inpainting by conditioning on particular subsets of inputs.

\subsection{Proofs}
\label{app:proofs}

\begin{lemma} 
\label{lemma:sigma-epsilon}
Let $X$ denote a normally distributed random variable,
$\sigma : \mathbb{R} \to \mathbb{R}_{\geq 0}$ denote a non-negative mapping  s.t. 
\[
 \forall x \in \mathbb{R}_{>0}, \; \exists \, \epsilon_x > 0 \;\; \text{s.t.} \;\; \sigma(x) > \epsilon_x.
\]
Then, $\exists \,\epsilon > 0 \;\; \text{s.t.} \;\; \mathbb{E}_X\left[ \sigma(x) \right]  > \epsilon$
\end{lemma}
\begin{proof}
Let $a,b \in\mathbb{R}$ and $0<a<b$. Then,
\begin{align}
     \mathbb{E}_X\left[ \sigma(x) \right] &= \int \sigma(x) f_X(x) \\
     \label{eq:lemma-proof-1}
     &\geq  \int_a^b \sigma(x) f(x) \\ 
     &\geq   \inf \{\sigma(x) \,|\, a\leq x\leq b\} \int_a^b f_X(x) \\
     \label{eq:lemma-proof-3}
     &= \epsilon^\prime \\
     &\geq \frac{\epsilon^\prime}{2} \\
     &= \epsilon
\end{align}
where \Eqref{eq:lemma-proof-1} holds because both $\sigma$ and $f_X$ are non-negative, \Eqref{eq:lemma-proof-3} holds because both $\sigma(x) > 0$ and $f_X (x) > 0$ for $x\in [a, b]$. Lastly, we have that $\epsilon > 0$ since $\epsilon^\prime > 0$.
\end{proof}

\begin{figure}[t]
  \centering
    \centering
    \input{diagrams/dmrg}
    \caption{(Reproduction of Figure~\ref{fig:dmrg}) Illustration of a single update step using the DMRG two site update algorithm used in~\cite{han2018unsupervised}. (1) cores $\tensor{g}^{(1)}$ and $\tensor{g}^{(2)}$ are merged (2) the loss is computed with respect to the merged fourth order tensor (3) the gradient is computed and used to update the fourth order tensor using automatic differentiation (4) the fourth order tensor is decomposed using SVD, then singular vectors are \emph{copied} into cores $\tensor{g}^{(1)}$ and  $\tensor{g}^{(2)}$.}
    \label{fig:dmrg-repr}
\end{figure}

\thmMpsSigmaInstability*

\begin{proof}
Let $\mathcal{R} = \{n \,|\, n\in\N, n\leq R\}$ denote a set of integers, then the expected value of $\Psi_\sigma(\vec y)$ is bounded below, since
    \begin{align}
        \E[\Psi_\sigma (\vec y)] 
        &= \E
        \left[ \sum_{\vec r \in \mathcal{R}^N} 
        \sigma \left( \tensor{g}_{r_1 r_2}^{(1)}[y_1] \right) 
        \cdots
        \sigma \left( \tensor{g}_{r_N r_{N+1}}^{(N)}[y_N] \right )
        \right] \\
        &=  \sum_{\vec r \in \mathcal{R}^N} 
        \E\Big[
        \sigma \left( \tensor{g}_{r_1 r_2}^{(1)}[y_1] \right)
        \Big]
        \cdots
        \E\Big[
        \sigma \left( \tensor{g}_{r_N r_{N+1}}^{(N)}[y_N] \right )
        \Big] \\
        \label{eq:thm3-proof-step-2}
        &> \sum_{\vec r \in \mathcal{R}^N}  \epsilon^{(r_1)} \cdots \epsilon^{(r_N)}  \\
        &\geq \sum_{\vec r \in \mathcal{R}^N}  \tilde\epsilon^N \\
        &= \epsilon R^N,
    \end{align}
where $\tilde\epsilon \triangleq \min \epsilon^{(r_1)} \cdots \epsilon^{(r_N)}$ and \Eqref{eq:thm3-proof-step-2} follows from Lemma~\ref{lemma:sigma-epsilon}a . Similarly, the normalization constant is bounded below,
    \begin{align}
        \E[Z_\sigma] 
        &= \E
        \left[ \sum_{\vec r \in \mathcal{R}^N} \sum_{\vec y \in \mathcal{Y}}
        \sigma \left( \tensor{g}_{r_1 r_2}^{(1)}[y_1] \right) 
        \cdots
        \sigma \left( \tensor{g}_{r_N r_{N+1}}^{(N)}[y_N] \right )
        \right] \\
        &=  \sum_{\vec r \in \mathcal{R}^N} \sum_{\vec y \in \mathcal{Y}}
        \E\Big[
        \sigma \left( \tensor{g}_{r_1 r_2}^{(1)}[y_1] \right)
        \Big]
        \cdots
        \E\Big[
        \sigma \left( \tensor{g}_{r_N r_{N+1}}^{(N)}[y_N] \right )
        \Big] \\
        &> \sum_{\vec r \in \mathcal{R}^N} \sum_{\vec y \in \mathcal{Y}}  \epsilon^{(1)} \cdots \epsilon^{(N)} \\
        &\geq \sum_{\vec r \in \mathcal{R}^N} \sum_{\vec y \in \mathcal{Y}} \tilde\epsilon^N \\
        &= \epsilon R^N D^N.
    \end{align}
\end{proof}

\thmMpsBMInstability*

\begin{proof}
We have that the expectation of $\Psi_\mathrm{BM} (\vec y)$ is zero since,
\begin{align}
\E[\Psi_\mathrm{BM} (\vec y)] 
        &= \E
        \left[ \sum_{\vec r \in \mathcal{R}^N} 
        \tensor{g}_{r_1 r_2}^{(1)}[y_1] 
        \cdots
        \tensor{g}_{r_N r_{N+1}}^{(N)}[y_N]
        \right] \\
        &=  \sum_{\vec r \in \mathcal{R}^N} 
        \E\Big[
        \tensor{g}_{r_1 r_2}^{(1)}[y_1]
        \Big]
        \cdots
        \E\Big[
        \sigma \left( \tensor{g}_{r_N r_{N+1}}^{(N)}[y_N] \right )
        \Big] \\
        &=  0.
\end{align}
Therefore, its variance is given by
\begin{align}
\E[\Psi_\mathrm{BM}(\vec y )^2] 
&=
\int_{\tensor{G}} 
\left(
\sum_{\vec r \in \mathcal{R}} 
\tensor{G}_{r_1 r_2}^{(1)}[y_1] \cdots  
\tensor{G}_{r_N r_{N+1}}^{(N)}[y_1]
\right)^2
\, f_\mathcal{G} (\tensor{G}) \;d\tensor{G} \\
\label{eq:thm4-proof-step-1}
&= \frac{C}{2} + 
\int_{\tensor{G} \in\mathcal{S}} 
\left(
\sum_{\vec r \in \mathcal{R}} 
\tensor{G}_{r_1 r_2}^{(1)}[y_1] \cdots  
\tensor{G}_{r_N r_{N+1}}^{(N)}[y_1]
\right)^2
\, f_\mathcal{G} (\tensor{G}) \;d\tensor{G} \\
\label{eq:thm4-proof-step-2} 
&\geq
\int_{\tensor{G}\in\mathcal{S}} \sum_{\vec r \in \mathcal{R}} \tensor{G}_{r_1 r_2}^{(1)}[y_1]^2  \cdots  \tensor{G}_{r_1 r_2}^{(N)}[y_1]^2 f_\mathcal{G} (\tensor{G}) f_\mathcal{G} (\tensor{G}) \;d\tensor{G} \\
&=
\int_{\tensor{G}\in\mathcal{S}} \sum_{\vec r \in \mathcal{R}} \sigma\left(\tensor{G}_{r_1 r2}^{(1)}[y_1]\right)  \cdots  \sigma\left(\tensor{G}_{r_1 r_2}^{(N)}[y_1]\right) f_\mathcal{G} (\tensor{G}) \;d\tensor{G} \\
&= 
\int_{\tensor{G}\in\mathcal{S}} \sum_{\vec r \in \mathcal{R}} \epsilon^{(r_1, r_2)} \cdots \epsilon^{(r_N, r_{N+1})} f_\mathcal{G} (\tensor{G}) \;d\tensor{G} \\
&= 
\frac{1}{2} \sum_{\vec r \in \mathcal{R}} \tilde{\epsilon} \\
&= 
\frac{1}{2}
\tilde{\epsilon}R^H \\
&= \epsilon R^H,
\end{align}
where $\mathcal S$ represents the infinite set consisting of all parameters $\tensor G\triangleq \{\tensor{G}^{(N)} \ldots \tensor G^{(N)}\}$ that result in a positive contraction at test point $\vec y$. Thus, \eqref{eq:thm4-proof-step-1} follows by symmetry of the distribution represented by a product of $N$ zero-mean independent gaussian random variables. 
\end{proof}

\end{document}